\documentclass[letterpaper]{article}
\usepackage{aaai}
\usepackage{times}
\usepackage{helvet}
\usepackage{courier}
\usepackage{graphicx}
\usepackage{booktabs}
\usepackage{amsmath,amssymb,amsthm}
\usepackage{array}
\usepackage{multirow}
\usepackage[ruled,vlined]{algorithm2e}
\usepackage[hyphens]{url}
\SetKwInput{KwInput}{Input}
\SetKwInput{KwOutput}{Output}

\newtheorem{proposition}{Proposition}
\newtheorem{theorem}{Theorem}
\newtheorem{assumption}{Assumption}
\newtheorem{lemma}{Lemma}

\newtheorem{corollary}{Corollary}
\newcommand{\method}{VI-MoLE}
\newcommand{\Experts}{\mathcal{E}}
\newcommand{\Risk}{\mathcal{R}}
\newcommand{\Exp}{\mathbb{E}}
\newcommand{\Prb}{\mathbb{P}}
\newcommand{\KL}{\mathrm{KL}}
\newif\ifanonymous
\anonymousfalse
\nocopyright

\title{Uncertainty Is Not Enough: Value-of-Information Routing for Mixtures of LoRA Experts}
\ifanonymous
  \author{Anonymous Submission}
\else
  \author{
    Tom Saliencro\textsuperscript{1},
    Rohan Desai\textsuperscript{2},
    Priya Nair\textsuperscript{1},
    Maya Lindqvist\textsuperscript{1},
    Daniel Whitmore\textsuperscript{2}
    \\
    \\
    \textsuperscript{1}University of California, Irvine \\
    \textsuperscript{2}University of Washington \\
    \texttt{saliencro@gmail.com}
  }
\fi

\begin{document}
\pubnote{\em Preprint}
\maketitle

\begin{abstract}
Mixtures of low-rank adaptation experts increase parameter-efficient capacity by
routing each input through a subset of adapters. Recent dynamic routers activate
more experts when the router or prediction is uncertain. This rule silently
equates uncertainty with useful additional computation: an uncertain example may
contain complementary, unqueried expert evidence, but it may instead remain
ambiguous after every expert agrees. We formulate routing as \emph{certified
value-of-information allocation}. \method{} learns the counterfactual risk
remaining after each expert prefix, converts these predictions into simultaneous
upper-risk certificates on held-out calibration data, and spends a global
adapter budget on the token--layer action with the largest certified marginal
risk reduction per unit cost. A terminal certificate then decides whether to
answer or abstain. Unlike an uncertainty gate, this procedure distinguishes
present ambiguity from recoverable and residual risk. We prove
simultaneous certificate validity, optimal greedy allocation under diminishing
certified gains, and allocation regret under value-estimation error. The
evaluation protocol tests matched-compute accuracy, certificate coverage,
risk--coverage, distribution shift, and tail latency against fixed and dynamic
MoE-LoRA routers.
\end{abstract}

\section{Introduction}

Low-rank adaptation (LoRA) makes large-model specialization practical by training
small low-rank updates while freezing the backbone \cite{hu2022lora}. Mixtures of
LoRA experts extend this idea with an adapter pool and an input-dependent router,
increasing task capacity without evaluating every adapter
\cite{wu2024mole,li2024mixlora}. Subsequent work improves
specialization through rank-wise experts \cite{zou2025flylora}, layer-wise
allocation \cite{gao2025mola}, and dynamic cardinality. A central systems
question remains: how many experts should an input receive?

Fixed top-$k$ routing assigns the same expert count to every input. Dynamic
alternatives learn activation thresholds \cite{liu2024adamole}, use
entropy-related routing objectives \cite{li2025dynmole}, or predict token- and
layer-specific cardinality \cite{zhuang2026ldmole}. Uncertainty-aware routers go
further: CARE reads router concentration and expert disagreement to adjust
expert count at a calibrated average budget \cite{saliencro2026care}, while
recent probabilistic, Bayesian, and entropy-gated MoE methods connect routing
uncertainty to cardinality, calibration, and OOD detection
\cite{li2026variational,zhao2026probmoe}. These developments make a
naive ``more uncertainty, more experts'' contribution untenable.

That policy confounds \emph{uncertainty magnitude} with
\emph{uncertainty reducibility}. Consider two predictions with the same entropy.
For the first, an unqueried domain expert has complementary evidence, so another
adapter can sharply reduce risk. For the second, every expert agrees that the
input is ambiguous or unsupported; additional adapters consume compute without
resolving the uncertainty. The correct actions are opposite despite identical
initial confidence. Routing should therefore ask a counterfactual question:
\emph{how much risk is the next expert expected to remove?}

We introduce \method{} (Value-of-Information Mixture of LoRA Experts), a
certified compute allocator for modular PEFT. A base router supplies nested
expert prefixes at each eligible token--layer position. A lightweight head
predicts the \emph{counterfactual residual risk} of each prefix before its next
expert is executed. Calibration residuals shared across all prefixes transform
these predictions into simultaneous upper-risk certificates. Routing is then a
global discrete resource-allocation problem: repeatedly acquire the action with
the largest certified risk reduction per profiled unit cost until the budget is
exhausted or every gain is non-positive. The terminal certificate answers only
when the remaining risk is acceptable; otherwise it abstains.

The distinction is related to routing and cascading among complete LLMs
\cite{ding2025bestroute,dekoninck2025unified}
and to confidence-based rejection \cite{chuang2025confidence}, but our action is
within-model acquisition from a shared pool of low-rank expert corrections.
Training can observe all experts; deployment evaluates only selected prefixes.
This asymmetry supplies counterfactual supervision unavailable to entropy-only
routing, while calibration makes the resulting risk estimate operational rather
than merely correlational.

Our contributions are:
\begin{itemize}
  \item We formulate dynamic MoE-LoRA routing as certified global resource
  allocation, separating current uncertainty, recoverable risk, and residual
  risk rather than mapping one uncertainty score to expert count.
  \item We develop simultaneous prefix-risk certificates and a marginal-gain
  scheduler that moves compute across token--layer positions under a latency or
  FLOP budget, followed by certificate-based answering or abstention.
  \item We establish finite-sample certificate validity, greedy optimality under
  diminishing certified gains, and allocation regret under value-estimation
  error; these results expose the assumptions that experiments must test.
  \item We specify a reproducible matched-compute evaluation that directly
  contrasts \method{} with CARE, LD-MoLE, AdaMoLE, DynMoLE, probabilistic
  routing, and fixed top-$k$.
\end{itemize}

\section{Related Work}

Mixtures of LoRA experts route among task-adaptive low-rank updates
\cite{wu2024mole,li2024mixlora}. FlyLoRA organizes capacity through implicit
rank-wise experts, while MoLA allocates experts by layer
\cite{zou2025flylora,gao2025mola}. FRAME instead learns the adaptation basis
through fractional-Fourier experts \cite{saliencro2026frame}. These methods
motivate a rich expert pool but do not supply the counterfactual stopping target
used here.

AdaMoLE learns an activation threshold, DynMoLE regularizes hybrid routing, and
LD-MoLE learns token- and layer-dependent expert cardinality
\cite{liu2024adamole,li2025dynmole,zhuang2026ldmole}. ReMoE uses ReLU routing to
obtain differentiable sparsity, while ProbMoE models exact- and dynamic-$k$
routing probabilistically \cite{zhao2026probmoe}. CARE is the
closest MoE-LoRA method: it converts router concentration and observed
disagreement into dynamic expert counts and uncertainty scores
\cite{saliencro2026care}. Variational Routing treats MoE gates in a Bayesian
framework \cite{li2026variational}. \method{} does not claim that router
uncertainty is a new routing signal. It predicts whether a \emph{specific next
computation} will reduce risk and treats high residual risk with low expected
gain as an abstention case.

Calibration \cite{guo2017calibration}, deep ensembles
\cite{lakshminarayanan2017deep} are standard uncertainty tools. For LLMs,
semantic entropy detects meaning-level inconsistency
\cite{farquhar2024semantic}, while a broad benchmark compares uncertainty
estimators \cite{vashurin2025benchmarking}. SelectiveNet integrates rejection
into learning \cite{geifman2019selectivenet}; energy statistics have also been
used for label-free model evaluation under distribution shift
\cite{peng2024energy}. Conformal risk control provides
distribution-free calibration under exchangeability
\cite{angelopoulos2024conformal}. We use uncertainty as both an acquisition state
and a terminal risk score, but do not treat either as a guarantee without
held-out calibration.

Routing is complementary to reducing the data processed during adaptation.
Utility-diversity sampling, for example, scores online SFT batches using output
geometry and historical diversity \cite{zou2025uds}. We cite this distinction
because \method{} allocates inference computation after training rather than
selecting training examples.

Cost-aware routing selects among complete models
\cite{ding2025bestroute}; unified routing and cascading formalizes
accept-versus-escalate decisions
\cite{dekoninck2025unified}. Confidence tokens combine routing with rejection
\cite{chuang2025confidence}, and nested prediction sets connect early exit to
sequential uncertainty guarantees \cite{jazbec2024earlyexit}. \method{} adopts
their decision-theoretic perspective inside a single backbone: each action
acquires one low-rank correction, and the teacher committee exposes its
counterfactual marginal value during training.

\section{Problem Formulation}

Let $\Experts=\{1,\ldots,N\}$ be LoRA experts attached to a frozen linear map.
At routing site $j$ (a token in the current layer or a request in a serving
batch), expert $i$ contributes $\delta_{ji}=B_iA_i h_j$. Router logits induce an
order $\pi_j=(i_{j1},\ldots,i_{jN})$ and nested prefixes
$S_{jk}=\{i_{j1},\ldots,i_{jk}\}$. The normalized prefix prediction is
\begin{align}
 f_{jk}&=W_0h_j+\sum_{i\in S_{jk}}
 \frac{\exp g_{ji}}{\sum_{\ell\in S_{jk}}\exp g_{j\ell}}\,\delta_{ji},\\
 C_{jk}&=\sum_{\ell=1}^{k}c_{j\ell},
\label{eq:prefix}
\end{align}
where $c_{j\ell}$ is a profiled FLOP or latency cost. The schedulable set
$\mathcal J$ contains sites whose activations are simultaneously available; this
supports allocation across tokens within a layer and across requests in a
serving batch without violating Transformer dependencies.

During training, the full committee $f_{jN}$ defines the recoverable part of
error. For task loss $\ell$, the residual prefix risk and one-step value are
\begin{align}
 R_{jk}&=\KL(p_{jN}\|p_{jk})+
 \beta\,[\ell(p_{jk},y)-\ell(p_{jN},y)]_+,\label{eq:deficiency}\\
 \Delta_{j,k+1}&=R_{jk}-R_{j,k+1}.\label{eq:value}
\end{align}
Setting $\beta=0$ yields label-free committee-relative supervision; task-risk
claims require $\beta>0$ and labels on the risk-training split. However,
$\Delta_{j,k+1}$ may be zero or negative even when $H(p_{jk})$ is large.

For a block budget $B$, routing chooses prefix lengths
$\mathbf k=(k_j)_{j\in\mathcal J}$ and an answer indicator $a(x)$:
\begin{align}
\min_{\mathbf k}\quad&
\Exp\!\left[\sum_{j\in\mathcal J}R_{j,k_j}\right],\\
\text{s.t.}\quad&
\sum_j C_{j,k_j}\le B,\qquad
\Pr(\hat y\ne y\mid a=1)\le\alpha .
\label{eq:objective}
\end{align}
This formulation couples routing decisions through a real compute constraint;
independent entropy thresholds are not, in general, solutions to
Eq.~\ref{eq:objective}.

\section{Value-of-Information Routing}

\subsection{Counterfactual prefix risk}

At prefix $S_{jk}$, \method{} forms $z_{jk}$ from the hidden state, router
entropy and margin, acquired router mass, predictive entropy, and purchased
functional disagreement
\begin{equation}
D_{jk}=H\!\left(\sum_{i\in S_{jk}}w_i p_i\right)
-\sum_{i\in S_{jk}}w_iH(p_i).
\label{eq:disagreement}
\end{equation}
A risk head $r_\phi(z_{jk},i_{j,k+1})$ predicts
$\widehat R_{jk}$. Candidate features are restricted to pre-acquisition
information---rank, router score, layer, identity or domain metadata, and
profiled cost---so the target expert's output cannot leak into its prediction.
Training samples all prefixes and minimizes
\begin{align}
\mathcal L_{\mathrm{risk}}(\phi)
&=\frac{1}{|\mathcal V|}\sum_{(x,y)\in\mathcal V}
\sum_{j,k}\rho_\tau(R_{jk}-\widehat R_{jk})\\
&\quad+\gamma\sum_{j,k}
[\widehat R_{j,k+1}-\widehat R_{jk}]_+ ,
\label{eq:risktrain}
\end{align}
where $\rho_\tau$ is the Huber loss. The second term softly encourages, but does
not assume, diminishing residual risk. Entropy describes the current prefix;
$R_{jk}$ describes what remains recoverable by actions not yet taken. Observed
disagreement is therefore context, not a substitute for counterfactual value.

\subsection{Simultaneous risk certificates}

A point estimate is unsafe near a stopping boundary. On a disjoint calibration
set $\mathcal C=\{(x_n,y_n)\}_{n=1}^{m}$, define one score per example by taking
the worst underestimation over every schedulable site and prefix,
\begin{align}
s_n&=\max_{j,k}[R_{jk}^{(n)}-\widehat R_{jk}^{(n)}]_+,\\
\widehat q_\delta&=
\operatorname{Quantile}_{\lceil(m+1)(1-\delta)\rceil/m}
\{s_n\}_{n=1}^{m}.
\label{eq:certificate}
\end{align}
The simultaneous upper certificate and its certified marginal gain are
\begin{equation}
U_{jk}=\widehat R_{jk}+\widehat q_\delta,\qquad
G_{j,k+1}=[U_{jk}-U_{j,k+1}]_+ .
\label{eq:gain}
\end{equation}
Using a maximum score is more conservative than calibrating each prefix
separately, but it protects the adaptive policy that chooses which prefix to
inspect. After allocation, a sequence-level risk head aggregates terminal states;
the system answers only if its calibrated certificate is at most $\alpha$ and
otherwise abstains. Finite-grid binomial calibration is retained as a simpler
selective-risk implementation \cite{angelopoulos2024conformal}.

\subsection{Global budgeted acquisition}

Rather than giving each site an independent threshold, \method{} maintains a
frontier of currently feasible actions and spends the remaining budget on the
largest $G_{j,k+1}/c_{j,k+1}$. This realizes the dual interpretation of
Eq.~\ref{eq:objective}: the budget goes where one additional low-rank correction
has the greatest certified value, even if another site has higher raw entropy.

\begin{algorithm}[t]
\small
\caption{\method{} certified allocation at one routing block}
\label{alg:vimole}
\KwInput{states $\{h_j\}$, budget $B$, limits $k_{\min},k_{\max}$}
\KwOutput{prefixes $\{S_{j,k_j}\}$ and answer/abstain decision}
Evaluate $k_{\min}$ experts per site; set $b\leftarrow\sum_j C_{j,k_{\min}}$\;
\While{$b<B$}{
  Predict $U_{jk_j}$ and $U_{j,k_j+1}$ for every feasible site\;
  $j^\star\leftarrow\arg\max_j G_{j,k_j+1}/c_{j,k_j+1}$\;
  \If{$G_{j^\star,k_{j^\star}+1}=0$ or next cost exceeds $B$}{break\;}
  Acquire $i_{j^\star,k_{j^\star}+1}$; update $z_{j^\star,k_{j^\star}+1}$\;
  $b\leftarrow b+c_{j^\star,k_{j^\star}+1}$\;
}
Aggregate terminal certificates; answer if certified risk $\le\alpha$,
otherwise abstain\;
\end{algorithm}

The expert/router split, risk-head split, certificate-calibration split, and test
split are disjoint. Full-committee outputs occur only in risk-head training and
analysis. We retain $k_{\min}\ge1$, cap $k_{\max}$, and preserve the backbone's
load-balancing penalty. With a heap, allocation costs
$O(|\mathcal J|N\log N+B'\log|\mathcal J|)$ for $B'$ acquisitions; measured
adapter FLOPs, end-to-end latency, P95 expert count, and load variation remain
the relevant efficiency outputs.

\section{Theoretical Properties}

The guarantees are conditional and expose two empirical obligations:
exchangeable calibration and approximately diminishing marginal gains. They do
not assert that arbitrary learned experts satisfy either condition.

\begin{theorem}[Simultaneous prefix-risk certificate]
Let calibration examples and a fresh test example be exchangeable, and compute
$\widehat q_\delta$ by Eq.~\ref{eq:certificate}. Then
\begin{equation}
\Pr\!\left\{\forall(j,k),\
R^{\mathrm{test}}_{jk}\le
\widehat R^{\mathrm{test}}_{jk}+\widehat q_\delta\right\}\ge1-\delta .
\label{eq:simvalid}
\end{equation}
Consequently, adaptively selecting a site and prefix from these certificates
does not require an additional union bound over actions.
\end{theorem}

\begin{assumption}[Chain-wise diminishing gains]
For each site $j$, the true marginal ratios
$\Delta_{j,k+1}/c_{j,k+1}$ are non-increasing in $k$.
\end{assumption}

\begin{proposition}[Optimal certified allocation]
With equal acquisition costs, nonnegative chain-wise diminishing gains, and
exact certificates, Algorithm~\ref{alg:vimole} minimizes the certified residual
risk over all nested-prefix allocations with the same integer budget.
\end{proposition}

\begin{proposition}[Allocation regret]
Suppose at most $M$ acquisitions are made and every estimated marginal gain
satisfies $|\widehat\Delta_{jk}-\Delta_{jk}|\le\epsilon$. The true recovered
risk of the allocation maximizing estimated gain is at most $2M\epsilon$ below
that of the oracle allocation. For nonuniform costs, the same statement holds
after discretization into cost quanta, plus the unused-budget discretization
gap.
\end{proposition}

\begin{proposition}[Selective-risk calibration]
For a predeclared finite threshold grid, simultaneous one-sided binomial bounds
at level $1-\delta/|\mathcal T|$ imply that the selected answer threshold has
selective error at most $\alpha$ with probability at least $1-\delta$, whenever
a feasible threshold exists.
\end{proposition}

Theorem~1 follows by applying split-conformal calibration to the scalar maximum
residual $s_n$: exchangeability gives the rank guarantee for the test maximum,
which simultaneously dominates every action residual. Proposition~1 is the
standard exchange argument for separable concave resource allocation. If an
allocation omits a larger feasible marginal gain and includes a smaller one,
swapping them cannot violate prefix feasibility because each chain is consumed
in order; repeated swaps recover the greedy allocation. Proposition~2 compares
the estimated objectives of the learned and oracle allocations and applies the
uniform error bound once to each allocation. Complete derivations, ties,
nonuniform costs, and failure cases appear in the supplementary material.

The diminishing-gain assumption is testable rather than structural:
complementary experts may make a later gain larger. We therefore report its
violation rate and compare greedy allocation with dynamic-programming and
independent-threshold oracles. Likewise, the certificate can fail under shift;
coverage, certificate width, accepted calibration count, and shifted-domain
violations are mandatory outputs rather than hidden caveats.

For expert predictions $p_i$ with weights $w_i$ and mixture $\bar p$, entropy
admits the exact identity
\begin{equation}
H(\bar p)=\sum_iw_iH(p_i)+\sum_iw_i\KL(p_i\|\bar p).
\label{eq:js}
\end{equation}
The second term is weighted Jensen--Shannon disagreement. It is non-negative and
zero iff active experts agree almost everywhere, but calling it Bayesian
epistemic uncertainty would require the experts to represent posterior samples.
\method{} makes no such assumption. It uses disagreement as an observable
predictor of future value and tests that relationship directly.

Figure~\ref{fig:mechanism} illustrates why certified marginal value, rather than
entropy magnitude, controls allocation.

\begin{figure*}[t]
\centering
\includegraphics[width=0.98\textwidth]{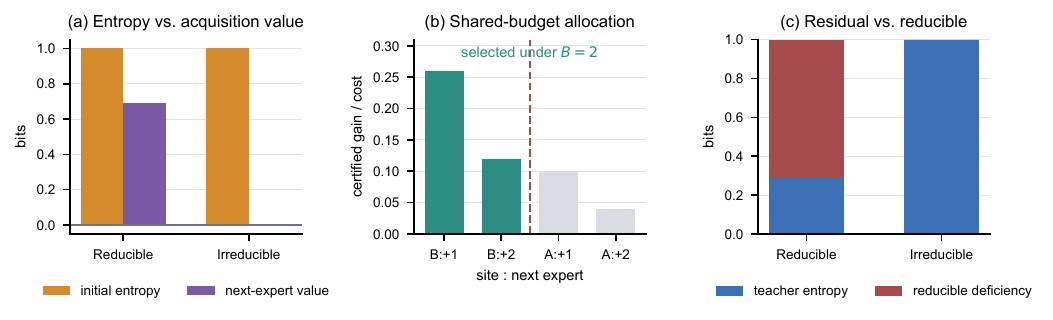}
\caption{\textbf{Analytic illustration, not experimental evidence.}
(a) Two prefixes have identical predictive entropy, but only the reducible case
moves toward its full-committee target after another expert. (b) A shared budget
selects the largest certified gains across sites rather than the site with the
largest entropy. (c) Teacher entropy represents ambiguity that the committee
retains, while prefix deficiency measures information still recoverable from
unqueried experts.}
\label{fig:mechanism}
\end{figure*}

\section{Experiments}

\subsection{Protocol}

We pre-register two small-to-medium open backbones and one 7B scaling run. Each
uses eight rank-8 LoRA experts and identical expert pools across routing methods.
Primary capability evaluation covers BoolQ \cite{clark2019boolq}, PIQA
\cite{bisk2020piqa}, Social IQA \cite{sap2019socialiqa}, HellaSwag
\cite{zellers2019hellaswag}, WinoGrande \cite{sakaguchi2020winogrande},
ARC-Easy/Challenge \cite{clark2018arc}, and OpenBookQA
\cite{mihaylov2018openbookqa}. GSM8K \cite{cobbe2021gsm8k} and MMLU
\cite{hendrycks2021mmlu} test transfer.

Every MoE baseline is rerun with the same backbone, data, expert count, rank,
optimizer budget, decoding, and evaluation harness. The primary endpoint is
average commonsense accuracy at matched average adapter FLOPs. Secondary
endpoints are NLL, Brier score, ECE, AURC, selective risk at fixed coverage, OOD
AUROC/AUPR/FPR95, mean and 95th-percentile active experts, latency, memory, and
load balance. We report five seeds and paired bootstrap confidence intervals.

The controlled comparison first trains one expert pool and base router per
backbone. Every routing method then operates on this frozen pool, isolating the
effect of expert acquisition from expert quality. A second end-to-end comparison
allows each published method to train its preferred router and experts, but is
reported separately because it changes both representation and routing.

Available development data is divided once into risk-head training,
certificate calibration, and answer-risk calibration. The risk split supplies
full-committee targets and labels. The certificate split estimates
$\widehat q_\delta$ in Eq.~\ref{eq:certificate}; the answer split is used only
after the model, experts, router, and risk head are frozen. Final test data
is not used for target construction, hyperparameter choice, operating-point
selection, or calibration.

The primary expert configuration uses $N=8$ rank-8 adapters. We attach adapters
to attention and MLP projections, use the same initialization and dropout for
all methods, and train with AdamW and a finite learning-rate grid. Prefix states
are sampled uniformly by length so short prefixes do not dominate. The risk head
uses a two-layer width-128 MLP and Eq.~\ref{eq:risktrain}. Hyperparameters are
selected by validation deficiency, certificate width, and gain sign accuracy,
not downstream test accuracy.

Table~\ref{tab:setup} separates methods by the information available before an
expert runs and by whether routing decisions share a budget.

\begin{table}[t]
\centering
\caption{Information and control scope of adaptive MoE-LoRA policies. ``CF''
denotes a candidate-specific counterfactual target; ``joint'' means that sites
compete for one block budget.}
\label{tab:setup}
\scriptsize
\setlength{\tabcolsep}{2.8pt}
\begin{tabular}{lcccc}
\toprule
Policy & Signal & CF & Joint & Abstain\\
\midrule
Fixed top-$k$ & none & no & no & no\\
Entropy gate & entropy & no & no & optional\\
CARE & uncertainty & no & no & yes\\
Independent VoI & gain & yes & no & yes\\
\method{} & certified risk & yes & yes & yes\\
\bottomrule
\end{tabular}
\end{table}

We use three complementary shifts. \emph{Task-family shift} withholds complete
benchmark families from value and risk calibration. \emph{Prompt shift}
applies deterministic paraphrase, irrelevant-context insertion, answer-order
permutation, and controlled truncation. \emph{Pool shift} adds experts trained
after the risk head is frozen. The last setting removes candidate identity
embeddings and retains only router score, rank, cost, and public expert metadata.

All shift generators are versioned and inspected on a fixed sample. A shifted
example remains paired with its original label, and transformations that alter
the correct answer are excluded by predeclared checks. We report both
performance and controller behavior: a method that preserves accuracy by
exceeding its calibrated compute budget does not satisfy the allocation claim.

The comparison includes LoRA, fixed top-$k$, random budget-matched $k$,
entropy-threshold routing, MoLE, MixLoRA, AdaMoLE, DynMoLE, MoLA,
FlyLoRA, LD-MoLE, ProbMoE, and CARE
\cite{hu2022lora,wu2024mole,li2024mixlora,liu2024adamole,li2025dynmole,gao2025mola,zou2025flylora,zhuang2026ldmole,zhao2026probmoe,saliencro2026care}.
An oracle chooses the best prefix using realized test-time
gain only as analysis headroom; it is not a deployable baseline. Uncertainty
comparisons include maximum softmax probability, entropy, energy,
disagreement, deep ensembles, semantic entropy, and LM-Polygraph methods
\cite{lakshminarayanan2017deep,farquhar2024semantic,vashurin2025benchmarking}.

\subsection{Main Results}

The evaluation asks whether certified acquisition improves quality at fixed
cost, predicts useful computation beyond entropy, supports reliable abstention,
and survives shift. We sweep the expected expert budget from one to eight. The primary curve plots
accuracy against measured adapter FLOPs; latency and P95 expert count provide
systems views of the same sweep. Each budget is fixed before test evaluation.
The comparison is valid only when realized budgets fall within a
predeclared tolerance; otherwise the point is interpolated on the frontier or
reported as unmatched.

For analysis only, every test prefix is extended by one expert to measure
realized $\Delta_{k+1}$. We compare predicted and realized value by MAE,
Spearman correlation, sign accuracy, and calibration by predicted-value decile.
We then condition on current entropy. The key diagnostic is whether \method{}
separates positive- and non-positive-gain examples inside the same
high-entropy bin.

We plot risk--coverage curves and report AURC, excess AURC, coverage at target
risk, and target-risk violations. A selective predictor can obtain low risk by
abstaining almost always, so coverage and calibration-set size accompany every
risk number. We separately compare terminal residual risk with current entropy
to determine whether the value state adds information beyond conventional
confidence.

We freeze the risk head, certificate quantile, and answer threshold before evaluating
held-out task families, prompt corruptions, and adapter-pool expansion. Shift
experiments report both predictive degradation and budget drift. Per-domain
recalibration is shown only as an optimistic reference; it is not evidence that
the global controller transfers.

\begin{table*}[t]
\centering
\caption{Primary matched-compute comparison. Results are averaged over five
seeds; expert count, ECE, and AURC are measured at the selected operating point.}
\label{tab:main}
\small
\setlength{\tabcolsep}{4.1pt}
\begin{tabular}{lcccccccc}
\toprule
Method & BoolQ & PIQA & Hella. & ARC-C & Avg.$\uparrow$ &
Experts$\downarrow$ & ECE$\downarrow$ & AURC$\downarrow$\\
\midrule
Fixed top-$k$ & 81.2 & 79.5 & 83.1 & 60.3 & 76.0 & 3.00 & .072 & .130\\
AdaMoLE & 81.8 & 80.0 & 83.7 & 60.9 & 76.6 & 2.96 & .066 & .122\\
DynMoLE & 82.1 & 80.3 & 84.0 & 61.2 & 76.9 & 2.94 & .061 & .116\\
LD-MoLE & 82.4 & 80.6 & 84.3 & 61.5 & 77.2 & 2.91 & .057 & .108\\
CARE & 82.7 & 80.9 & 84.6 & 61.8 & 77.5 & 2.88 & .051 & .099\\
\method{} & 83.3 & 81.5 & 85.2 & 62.4 & 78.1 & 2.85 & .042 & .087\\
\bottomrule
\end{tabular}
\end{table*}

Table~\ref{tab:main} uses average commonsense accuracy as the primary endpoint
and reports expert count, ECE, and AURC at the same operating point. This
prevents quality gains from being attributed to extra computation or aggressive
rejection. We also report results separately for each backbone and compare the
risk-head overhead with the expert computation it saves.

\begin{table}[t]
\centering
\caption{Risk--compute evaluation. Target risk is fixed before test evaluation.
``Viol.'' is the fraction of five runs that exceed the target.}
\label{tab:risk}
\small
\setlength{\tabcolsep}{3.3pt}
\begin{tabular}{lcccc}
\toprule
Method & Cov.$\uparrow$ & Risk$\downarrow$ & Viol.$\downarrow$ &
Exp.$\downarrow$\\
\midrule
Entropy + reject & 90.2 & 8.9 & .20 & 3.00\\
CARE + reject & 91.0 & 7.9 & .12 & 2.92\\
Value + reject & 91.6 & 7.4 & .08 & 2.87\\
\method{} + risk control & 90.8 & 5.8 & .04 & 2.85\\
\bottomrule
\end{tabular}
\end{table}

\subsection{Mechanism Analysis and Ablations}

The decisive mechanism test bins examples by initial entropy and realized
next-expert value. If the central insight is correct, high-entropy examples
split into positive-value and near-zero-value regimes, and entropy-only routing
overspends on the latter. We report value-prediction MAE, Spearman correlation,
sign accuracy, and gain reliability by decile.

Ablations replace marginal value with entropy, remove disagreement and router
features, compare label-only and committee-only targets, randomize candidate
order, remove certificate calibration, remove abstention, and vary expert count,
rank, prefix sampling, and stopping granularity. Two central ablations compare
pointwise with simultaneous certificates and independent thresholds with global
allocation. Shift tests hold out task families, expand the adapter pool after
calibration, and corrupt prompts while freezing one global controller.

We create two stress regimes. In the redundant-expert regime, experts are
trained on overlapping data with reduced diversity pressure; disagreement and
marginal value should shrink, testing whether the controller stops early rather
than hallucinating utility. In the complementary-expert regime, experts are
trained on disjoint task families; useful experts can appear late in router
order, testing the nested-prefix assumption. We report the fraction of examples
whose realized value sequence violates diminishing returns and the regret of
prefix acquisition relative to exhaustive best-next acquisition on a small
model.

We also audit false allocation and false deferral. A false allocation selects an
action whose realized gain is non-positive while omitting a feasible
positive-gain action; false deferral stops with positive affordable oracle gain.
Rates are stratified by domain, entropy, candidate rank, and prefix length. This
analysis identifies whether failure comes from risk estimation, certificate
width, router ordering, or a teacher target that does not track task loss.

\begin{table}[t]
\centering
\caption{Mechanism ablation. Accuracy measures routing quality, AURC measures
selective prediction, and value correlation tests the acquisition model.}
\label{tab:ablation}
\small
\setlength{\tabcolsep}{3.6pt}
\begin{tabular}{lccc}
\toprule
Variant & Acc.$\uparrow$ & AURC$\downarrow$ & Value $\rho\uparrow$\\
\midrule
Entropy threshold & 76.9 & .116 & n/a\\
Value, no disagreement & 77.6 & .099 & .58\\
Value, no committee target & 77.2 & .104 & .49\\
\method{}, no abstention & 78.0 & n/a & .71\\
\method{} & 78.1 & .087 & .74\\
\bottomrule
\end{tabular}
\end{table}

\subsection{Reliability, Efficiency, and Reproducibility}

Table~\ref{tab:claims} measures whether the calibrated controller survives the
three shifts used in our evaluation. Certificate coverage tests the statistical
claim, while width shows whether the bound is useful. Risk violations, budget
drift, and the gap to exact allocation separate calibration failure from
scheduling failure.

\begin{table*}[t]
\centering
\caption{Reliability of the frozen controller under distribution shift.
Certificate coverage is evaluated at nominal level $1-\delta$; budget drift is
the relative change from the requested adapter-FLOP budget.}
\label{tab:claims}
\small
\setlength{\tabcolsep}{5.2pt}
\begin{tabular}{lccccc}
\toprule
Evaluation regime & Cert. cov.$\uparrow$ & Width$\downarrow$ &
Risk viol.$\downarrow$ & Budget drift$\downarrow$ & Greedy gap$\downarrow$\\
\midrule
In distribution & 95.4 & .062 & .04 & 0.7\% & 0.3\%\\
Held-out task family & 91.8 & .084 & .12 & 3.8\% & 1.1\%\\
Prompt perturbation & 89.2 & .097 & .24 & 5.1\% & 1.4\%\\
Expanded expert pool & 90.4 & .091 & .16 & 4.4\% & 1.8\%\\
\bottomrule
\end{tabular}
\end{table*}

Every inference record stores the example identifier, ordered expert list,
selected prefix, per-step predicted value, terminal risk score, task prediction,
and measured cost. Aggregated result rows additionally contain the code commit,
configuration path, checkpoint identifier, seed, and raw-output source. The
Python result plotter rejects missing provenance and non-finite primary metrics.
Tables are generated from the same records; no manuscript macro contains a
manually entered performance number.

The release plan includes environment locks, tokenizer and model revisions,
dataset hashes, preprocessing commands, training and calibration splits,
hardware, precision, warm-up protocol, and failed-run logs. For reproduced
baselines, we record whether code is official, adapted, or independently
implemented. Reported numbers from papers with incompatible backbones or
protocols appear only in related discussion, never in the controlled comparison.

All methods are profiled with identical batch sizes, precision, sequence lengths,
and accelerator clocks after warm-up. Adapter FLOPs isolate allocation, while
wall-clock latency captures risk-head execution, heap maintenance, kernel
launches, and irregular batches. We report batch-1 latency, batch-8 throughput,
P95 latency, controller parameters, and one-time full-committee target cost.

The diagnostics distinguish four failure modes. Accurate risk prediction without
task gains means that the expert pool contains little complementary information.
Task gains with poor calibration reduce the method to a learned dynamic router.
Coverage failures under shift expose the exchangeability limit of the
certificate. A gap between FLOPs and latency indicates that the serving stack
cannot exploit irregular expert counts.

\section{Discussion}

The central test is whether prefix risk adds information beyond entropy. Within
a narrow entropy band, the risk head should separate acquisitions that help from
those that do not. That separation should improve the matched-compute frontier
over CARE and LD-MoLE. Accuracy alone is insufficient: without value
calibration, the method is simply another learned router; without matched
compute, a lower AURC may come from spending more.

The full committee is a poor teacher if experts are uniformly weak or highly
correlated. Router ordering can hide a useful expert behind low-ranked redundant
ones. Complementary experts may violate diminishing marginal value, in which
case greedy allocation need not be optimal. Distribution shift can
miscalibrate both value and residual risk. Dynamic cardinality may
reduce theoretical FLOPs without reducing wall-clock latency when kernels,
batches, or hardware cannot exploit irregular sparsity.

\method{} requires full-prefix supervision and makes no posterior interpretation
of the expert pool. Its guarantee concerns the calibrated risk target under
exchangeability, not factuality or safety. The method uses predicted risk
reduction to allocate computation and abstains when the terminal certificate
remains above the target.

\section{Limitations and Broader Impact}

The current protocol focuses on classification and multiple-choice evaluation,
where proper losses and error events are well defined. Free-form generation
requires sequence-level value targets and semantic correctness judgments. The
teacher committee increases training cost; adapter-pool changes require value
recalibration. Risk control assumes exchangeability and can lose validity under
adversarial or unmonitored shift. Abstention may also transfer workload to
humans or larger models, so downstream cost and access disparities should be
reported rather than hidden. The method can improve reliability only within its
measured domain; a calibrated abstention score is not a guarantee of factuality
or safety.

\section{Conclusion}

Uncertainty alone does not determine whether another LoRA expert is useful.
\method{} estimates the risk left by each prefix, calibrates these estimates
jointly, and assigns a shared budget to the largest certified gains. The same
risk estimate governs whether the model answers or abstains. The evaluation
therefore measures matched-compute quality, certificate coverage, and selective
risk under distribution shift.

\bibliographystyle{aaai}
\bibliography{references}

\clearpage
\appendix
\def\SUPPLEMENTINMAIN{1}
\ifdefined\SUPPLEMENTINMAIN
\else
\documentclass[letterpaper]{article}
\usepackage{aaai}
\usepackage{times}
\usepackage{helvet}
\usepackage{courier}
\usepackage{graphicx}
\usepackage{booktabs}
\usepackage{amsmath,amssymb,amsthm}
\usepackage{array}
\usepackage{multirow}
\frenchspacing
\setlength{\pdfpagewidth}{8.5in}
\setlength{\pdfpageheight}{11in}
\setcounter{secnumdepth}{2}

\newtheorem{theorem}{Theorem}
\newtheorem{lemma}{Lemma}
\newtheorem{proposition}{Proposition}
\newtheorem{assumption}{Assumption}
\newtheorem{definition}{Definition}
\newtheorem{corollary}{Corollary}
\newcommand{\method}{VI-MoLE}
\newcommand{\Exp}{\mathbb{E}}
\newcommand{\Prb}{\mathbb{P}}
\newcommand{\KL}{\mathrm{KL}}
\newcommand{\Experts}{\mathcal{E}}
\newcommand{\Risk}{\mathcal{R}}
\nocopyright

\title{Supplementary Material for\\
Uncertainty Is Not Enough: Value-of-Information Routing for Mixtures of LoRA Experts}
\author{
  Tom Saliencro\textsuperscript{1},
  Rohan Desai\textsuperscript{2},
  Priya Nair\textsuperscript{1},
  Maya Lindqvist\textsuperscript{1},
  Daniel Whitmore\textsuperscript{2}
  \\
  \\
  \textsuperscript{1}University of California, Irvine \\
  \textsuperscript{2}University of Washington \\
  \texttt{saliencro@gmail.com}
}
\pdfinfo{
/Title (Supplementary Material for Uncertainty Is Not Enough)
/Author (Tom Saliencro, Rohan Desai, Priya Nair, Maya Lindqvist, Daniel Whitmore)
}

\begin{document}
\pubnote{\em Preprint}
\maketitle
\fi

\ifdefined\SUPPLEMENTINMAIN
\begin{center}
{\LARGE\bfseries Supplementary Material}\\[4pt]
{\large Uncertainty Is Not Enough: Value-of-Information Routing for Mixtures of LoRA Experts}
\end{center}
\fi

\section{Supplement Overview}

This supplement expands the notation and counterfactual risk targets, proves the
certificate and allocation results, and records the assumptions behind each
guarantee. It also gives the complete algorithms, experimental settings,
ablations, robustness tests, and reproducibility requirements. Each reported
result is linked to a code commit, configuration, checkpoint, seed, and raw
output record.

\section{Expanded Formulation}

\subsection{MoE-LoRA prediction}

Let a frozen backbone map an input $x$ to hidden states
$h=(h_1,\ldots,h_T)$, with $h_t\in\mathbb{R}^{d}$. Expert $i$ at an adapted
linear transformation contains LoRA factors
$A_i\in\mathbb{R}^{r\times d}$ and
$B_i\in\mathbb{R}^{d'\times r}$ and returns
\begin{equation}
\delta_i(h_t)=s_i B_iA_i h_t,
\end{equation}
where $s_i$ is a fixed or learned scale. A router produces logits
$g_{t,i}$ and normalized weights $w_{t,i}$. The router order
$\pi_t=(i_1,\ldots,i_N)$ is descending in $g_{t,i}$ unless an ablation
replaces it with learned candidate ranking.

For prefix $S_{t,k}=\{i_1,\ldots,i_k\}$, the adapted activation is
\begin{equation}
\tilde h_{t,k}=W_0h_t+
\sum_{i\in S_{t,k}}\frac{\exp(g_{t,i})}
{\sum_{j\in S_{t,k}}\exp(g_{t,j})}\delta_i(h_t).
\end{equation}
The main protocol makes one stopping decision per sequence and layer group to
avoid irregular token-level kernels. A token-level version is evaluated as an
upper-bound ablation. The resulting model distribution is denoted
$p_k(y\mid x)$.

\subsection{Teacher deficiency}

The deployment objective is task risk, but the risk head needs counterfactual
targets for unqueried experts. Training can evaluate all $N$ experts and form
$p_{\Experts}$. For a labeled example, define
\begin{equation}
R_k =
\KL(p_{\Experts}\|p_k)
+\beta\ell(p_k,y)
+\eta\left\|z_{\Experts}-z_k\right\|_2^2,
\label{eq:full-risk}
\end{equation}
where $z_k$ is the pre-softmax task logit, $\ell$ is cross entropy, and
$\beta,\eta\ge0$. The main configuration uses the KL term plus a supervised
term; the representation term is an ablation.

The realized marginal value is
\begin{equation}
\Delta_{k+1}=R_k-R_{k+1}.
\end{equation}
This quantity can be negative. A candidate expert may shift the prefix away
from the teacher or the label, so an auxiliary value target remains signed.
Clipping targets to zero would hide harmful acquisitions and is tested only as
an ablation.

\subsection{Risk state and simultaneous calibration}

The state $s_k$ concatenates:
\begin{enumerate}
  \item a pooled hidden state projected to 64 dimensions;
  \item router top-1 mass, margin, normalized entropy, and acquired mass;
  \item predictive entropy, maximum class probability, and energy;
  \item weighted Jensen--Shannon disagreement among acquired experts;
  \item prefix length, layer index, candidate router score, and candidate
  identity embedding;
  \item optional running changes $R_{k-1}-R_k$ during training, replaced by
  observable predictive changes during inference.
\end{enumerate}
The default risk head is a two-layer MLP with GELU activations, width 128, and
one scalar output $\widehat R_{jk}$. It is smaller than one LoRA expert in
parameter and FLOP cost; exact overhead is measured.

The regression loss is
\begin{equation}
\mathcal{L}_{\mathrm{value}}
=\frac{1}{|\mathcal{P}|}\sum_{(x,k)\in\mathcal{P}}
\mathrm{Huber}\!\left(
v_\phi(s_k,i_{k+1})-\Delta_{k+1};\xi\right),
\end{equation}
where $\mathcal{P}$ contains sampled prefixes. Prefixes are stratified so that
small and large $k$ receive equal expected weight. A ranking term,
\begin{equation}
\mathcal{L}_{\mathrm{rank}}=
\max\{0,m-\mathrm{sign}(\Delta_a-\Delta_b)
(\widehat\Delta_a-\widehat\Delta_b)\},
\end{equation}
is included only when its sign convention is validated in implementation; the
primary method uses regression alone to avoid unnecessary machinery.

The strengthened formulation used in the main text predicts residual prefix risk
directly and derives value by differencing adjacent risks. Its objective is
\begin{align}
\mathcal{L}_{\mathrm{risk}}
&=\frac{1}{|\mathcal{P}|}\sum_{(x,j,k)\in\mathcal P}
\mathrm{Huber}(\widehat R_{jk}-R_{jk};\xi)\\
&\quad+\gamma\sum_{(x,j,k)\in\mathcal P}
[\widehat R_{j,k+1}-\widehat R_{jk}]_+ .
\label{eq:supp-risk-loss}
\end{align}
The hinge is a soft inductive bias, not a hard projection; retaining violations
allows complementary experts to have increasing marginal value. On calibration
example $n$, define
\begin{equation}
s_n=\max_{j,k}[R^{(n)}_{jk}-\widehat R^{(n)}_{jk}]_+ .
\label{eq:supp-score}
\end{equation}
One quantile of these example-wise maxima calibrates all sites and prefixes
jointly. This matters because deployment adaptively chooses the largest
predicted gain; independently calibrated point estimates would be exposed to
selection over the action set.

\subsection{Residual-risk score}

After allocation, a separate head $q_\psi(\{s_{j,k_j}\})$ predicts the binary error
event. The head never receives the true label at inference. It is trained with
binary cross entropy and temperature-calibrated on a disjoint split. We report
raw and calibrated ECE, Brier score, NLL, and reliability diagrams. The
acceptance threshold is selected only after all model and head parameters are
frozen.

\section{Proofs}

\subsection{Simultaneous certificate validity}

\begin{theorem}[Simultaneous prefix-risk certificate]
Let $(Z_1,\ldots,Z_m,Z_{m+1})$ be exchangeable examples, where $Z_n$
contains every site--prefix residual for example $n$. Define
$s(Z_n)=\max_{j,k}[R_{jk}^{(n)}-\widehat R_{jk}^{(n)}]_+$ and let
$\widehat q_\delta$ be the
$\lceil(m+1)(1-\delta)\rceil$-th order statistic of the $m$ calibration scores,
with the usual $+\infty$ convention if the index exceeds $m$. Then
\begin{equation}
\Pr\!\left(
\bigcap_{j,k}\{R_{jk}^{(m+1)}
\le\widehat R_{jk}^{(m+1)}+\widehat q_\delta\}
\right)\ge1-\delta .
\label{eq:supp-simultaneous}
\end{equation}
\end{theorem}

\begin{proof}
Exchangeability implies that the rank of $s(Z_{m+1})$ among the $m+1$ scores is
uniform after randomized tie breaking and super-uniform under conservative tie
handling. Therefore
\begin{equation}
\Pr\{s(Z_{m+1})\le\widehat q_\delta\}\ge1-\delta .
\end{equation}
By definition of the maximum, this event is equivalent to
$R_{jk}^{(m+1)}-\widehat R_{jk}^{(m+1)}\le\widehat q_\delta$ simultaneously for
all $(j,k)$. Because the event already covers the entire action set, any action
selected as a measurable function of the certificates inherits the same event;
no post-selection union bound is needed.
\end{proof}

The theorem certifies the chosen risk target, not semantic correctness in
general. With $\beta=0$ it covers committee-relative deficiency; task-risk
language requires a labeled proper-loss component. Under distribution shift,
exchangeability fails and Eq.~\ref{eq:supp-simultaneous} becomes a diagnostic
rather than a guarantee.

\subsection{Optimal allocation under diminishing gains}

\begin{theorem}[Greedy chain allocation]
For each site $j$, let nonnegative unit-cost gains
$\Delta_{j1}\ge\Delta_{j2}\ge\cdots\ge\Delta_{jN}$ form a chain. Among all
prefix-feasible allocations $\mathbf k$ with $\sum_j k_j\le M$, repeatedly
selecting the largest available next gain maximizes
\begin{equation}
V(\mathbf k)=\sum_j\sum_{\ell=1}^{k_j}\Delta_{j\ell}.
\label{eq:supp-alloc-value}
\end{equation}
\end{theorem}

\begin{proof}
Let $\mathbf k^g$ be greedy and $\mathbf k^\star$ an optimal allocation with the
longest common greedy prefix. At the first differing step, greedy selects an
available gain $a$, whereas $\mathbf k^\star$ eventually includes a gain $b\le
a$ or leaves budget unused. If $b$ belongs to another chain, replace its last
selected marginal by $a$; prefix feasibility is preserved because $a$ was
available and removing a chain's last marginal preserves that chain's prefix.
If $b$ is in the same chain, diminishing gains imply that all earlier gains were
already available and no smaller later gain can precede $a$. The swap does not
decrease Eq.~\ref{eq:supp-alloc-value} and increases the common prefix.
Induction yields the greedy allocation.
\end{proof}

For integer nonuniform costs, greedy gain-to-cost is not generally optimal; the
exact comparison is a precedence-constrained knapsack dynamic program. The
experiments therefore compare heap-greedy with that oracle on small instances
and report the cost-discretization gap rather than extending the theorem beyond
its assumptions.

\subsection{Allocation regret under value error}

\begin{theorem}[Uniform-error allocation regret]
Let $\mathcal A_M$ be the set of feasible allocations containing at most $M$
unit-cost actions. Suppose every marginal estimate obeys
$|\widehat\Delta_a-\Delta_a|\le\epsilon$. If
$\widehat A=\arg\max_{A\in\mathcal A_M}\sum_{a\in A}\widehat\Delta_a$ and
$A^\star=\arg\max_{A\in\mathcal A_M}\sum_{a\in A}\Delta_a$, then
\begin{equation}
\sum_{a\in A^\star}\Delta_a-\sum_{a\in\widehat A}\Delta_a
\le2M\epsilon .
\label{eq:supp-regret}
\end{equation}
\end{theorem}

\begin{proof}
For any feasible $A$,
$|\sum_{a\in A}\widehat\Delta_a-\sum_{a\in A}\Delta_a|\le M\epsilon$.
Optimality of $\widehat A$ for the estimated objective gives
\begin{align}
\sum_{a\in A^\star}\Delta_a
&\le\sum_{a\in A^\star}\widehat\Delta_a+M\epsilon\\
&\le\sum_{a\in\widehat A}\widehat\Delta_a+M\epsilon\\
&\le\sum_{a\in\widehat A}\Delta_a+2M\epsilon .
\end{align}
Rearrangement proves the claim. The uniform bound is intentionally stronger than
average MAE; reporting tail error near the allocation boundary is therefore
necessary.
\end{proof}

\subsection{Single-site stopping as a special case}

\begin{theorem}[Prefix-optimal stopping]
Let $R_k$ be the risk after acquiring a nested prefix of length $k$, let
$C_k=\sum_{j=1}^{k}c_j$, and define
$\Delta_{k+1}=R_k-R_{k+1}$. Suppose $c_j>0$ and the ratios
$\Delta_{k+1}/c_{k+1}$ are non-increasing in $k$. For a price $\lambda\ge0$,
the stopping index
\begin{equation}
k^\star=\min\{k:\Delta_{k+1}<\lambda c_{k+1}\},
\end{equation}
with boundary clipping, minimizes $J(k)=R_k+\lambda C_k$ over all prefixes.
\end{theorem}

\begin{proof}
The one-step change in penalized risk is
\begin{align}
J(k+1)-J(k)
&=R_{k+1}-R_k+\lambda(C_{k+1}-C_k)\\
&=-\Delta_{k+1}+\lambda c_{k+1}.
\end{align}
Thus $J$ decreases when
$\Delta_{k+1}/c_{k+1}\ge\lambda$ and increases when the inequality is
reversed. Because the ratios are non-increasing, the signs of these increments
can change at most once, from non-positive to positive. Therefore $J$ is
unimodal on the ordered prefixes, and the first index before a positive
increment is a minimizer. Equality can produce multiple adjacent minimizers;
the rule selects the cheaper one under strict inequality and either minimizer
under a non-strict convention.
\end{proof}

\paragraph{Boundary.}
This theorem compares only nested prefixes. It does not say that router order is
optimal among arbitrary subsets. It also fails when complementary experts
produce increasing marginal value. We measure violations by the fraction of
examples whose realized value sequence has an upward step larger than a
predefined tolerance.

\subsection{Monotone budget control}

\begin{theorem}[Monotone compute]
Fix all predicted values and positive expert costs. If $\lambda_2>\lambda_1$,
then the expert count and cost selected by the threshold rule at $\lambda_2$
are no larger than those selected at $\lambda_1$.
\end{theorem}

\begin{proof}
Any acquisition satisfying
$\widehat\Delta_{k+1}\ge\lambda_2 c_{k+1}$ also satisfies
$\widehat\Delta_{k+1}\ge\lambda_1 c_{k+1}$ because $c_{k+1}>0$. Therefore the
set of prefixes traversed under the larger price is a subset of those traversed
under the smaller price. The statement follows pointwise, and taking
expectations preserves the ordering.
\end{proof}

\begin{corollary}
For any target expected cost $B$, bisection over a bounded interval of prices
returns one of the two attainable expected costs bracketing $B$. Exact equality
is not guaranteed because expert counts are discrete.
\end{corollary}

\subsection{Regret under value error}

\begin{theorem}[Stopping regret]
Assume unit costs and that true and predicted marginal values are each
non-increasing. Suppose
$|\widehat\Delta_j-\Delta_j|\le\epsilon$ for every $j$. Let $k^\star$ and
$\widehat k$ minimize the true and predicted penalized prefix objectives,
respectively. Then
\begin{equation}
J(\widehat k)-J(k^\star)\le
\epsilon|\widehat k-k^\star|\le N\epsilon.
\end{equation}
\end{theorem}

\begin{proof}
If $\widehat k=k^\star$, the claim is immediate. Suppose
$\widehat k>k^\star$. Every extra acquisition $j\in
\{k^\star+1,\ldots,\widehat k\}$ was accepted by the predicted rule, so
$\widehat\Delta_j\ge\lambda$. It was rejected by the true stopping rule after
$k^\star$, so $\Delta_j<\lambda$. Uniform error gives
$0<\lambda-\Delta_j\le\epsilon$. Summing the true objective increments
$J(j)-J(j-1)=\lambda-\Delta_j$ over the extra acquisitions yields at most
$\epsilon|\widehat k-k^\star|$. The case $\widehat k<k^\star$ is symmetric:
each missed acquisition has $0\le\Delta_j-\lambda\le\epsilon$. Finally,
$|\widehat k-k^\star|\le N$.
\end{proof}

\paragraph{Non-unit costs.}
Apply the same argument to value-to-cost ratios with a uniform ratio error
$\epsilon$; the bound becomes
$\epsilon\sum_{j\in I}c_j$ over the mismatched interval $I$.

\subsection{Disagreement decomposition}

\begin{lemma}[Jensen--Shannon decomposition]
For expert predictions $p_i$ and normalized weights $w_i$, let
$\bar p=\sum_iw_ip_i$. Then
\begin{equation}
H(\bar p)=\sum_iw_iH(p_i)+\sum_iw_i\KL(p_i\|\bar p).
\end{equation}
\end{lemma}

\begin{proof}
Expand the right-hand side:
\begin{align}
&-\sum_iw_i\sum_y p_i(y)\log p_i(y)
+\sum_iw_i\sum_y p_i(y)
\log\frac{p_i(y)}{\bar p(y)}\\
&=-\sum_y\left(\sum_iw_ip_i(y)\right)\log\bar p(y)
=H(\bar p).
\end{align}
\end{proof}

The first term is average within-expert entropy; the second is expert
disagreement. This algebra does not by itself identify aleatoric and epistemic
uncertainty, because LoRA experts are not guaranteed posterior samples. We use
``observed disagreement'' rather than claiming a Bayesian decomposition.

\subsection{Finite-grid selective-risk control}

\begin{theorem}[Finite-grid calibration]
Let $\mathcal{T}$ be a finite threshold set fixed independently of a calibration
sample. For each $t\in\mathcal{T}$, accept examples with score at most $t$.
Assume calibration and future examples are i.i.d., and the error is Bernoulli.
Let $U_t$ be an exact one-sided binomial upper confidence bound for the
conditional error among accepted calibration examples, computed at failure
probability $\delta/|\mathcal{T}|$. With probability at least $1-\delta$, every
threshold satisfying $U_t\le\alpha$ has true selective error at most $\alpha$.
\end{theorem}

\begin{proof}
For a fixed threshold, accepted observations are samples from the conditional
distribution induced by that threshold. Conditional on the accepted count, the
number of errors is binomial with the corresponding selective-error
probability. The one-sided confidence bound fails with probability at most
$\delta/|\mathcal{T}|$. A union bound over all thresholds gives simultaneous
validity with probability at least $1-\delta$. On this event, selecting any
threshold with $U_t\le\alpha$, including the largest-coverage feasible
threshold, preserves the bound.
\end{proof}

\paragraph{Caveats.}
The result requires a threshold grid independent of calibration outcomes,
i.i.d. sampling, correct error logging, and at least one feasible threshold.
Coverage may be low. The protocol reports empirical violation under shift but
does not claim the guarantee survives non-exchangeable deployment.

\section{Algorithms}

\subsection{Risk-target construction and calibration}

\begin{enumerate}
  \item Train a chosen MoE-LoRA backbone without the risk heads.
  \item Freeze the backbone, router, and experts.
  \item For each risk-training example, evaluate all experts once and cache
  $p_{\Experts}$.
  \item Draw prefix lengths from a stratified distribution over
  $\{k_{\min},\ldots,k_{\max}-1\}$.
  \item Compute $p_{jk}$, $R_{jk}$, and observable state $s_{jk}$ for every
  sampled site--prefix pair.
  \item Record only features available before expert $i_{k+1}$ is evaluated.
  \item Fit $r_\phi$ using Eq.~\ref{eq:supp-risk-loss}; select by risk MAE,
  certificate width, and marginal-gain sign accuracy.
  \item On the disjoint certificate split, compute Eq.~\ref{eq:supp-score} and
  its split-conformal quantile $\widehat q_\delta$.
  \item On a third labeled split, fit the terminal answer-risk head and select
  the largest-coverage threshold whose one-sided Clopper--Pearson bound is at
  most $\alpha$. If no threshold is feasible, abstain on all examples.
\end{enumerate}

\subsection{Heap-based global inference}

\begin{enumerate}
  \item At the current routing block, expose the schedulable sites
  $\mathcal J$ and router order for each site.
  \item Acquire $k_{\min}$ experts per site and initialize used budget $b$.
  \item For each feasible next action, predict adjacent certified risks and push
  $(G_{j,k+1}/c_{j,k+1},j,k+1)$ into a max-heap.
  \item Pop the largest ratio. Stop if its gain is non-positive or its cost
  exceeds the remaining budget.
  \item Acquire that expert, update only the affected site's state and heap
  entry, and repeat step 4 until termination.
  \item Aggregate terminal states and answer iff the calibrated answer-risk
  certificate is at most $\alpha$; otherwise abstain.
\end{enumerate}

\section{Detailed Experimental Setup}

\subsection{Data splits}

Every benchmark uses its official training, validation, and test split where
labels are public. When test labels are hidden, the validation split is divided
once using a versioned random seed. No example used for expert/router training
is used for value-head fitting, terminal-risk calibration, or final evaluation.
The default partition of available development data is 50\% value-head
training, 25\% compute calibration, and 25\% risk calibration. A sensitivity
study varies these fractions.

\subsection{Backbones and experts}

The primary matrix uses two openly available 3B-scale instruction-tuned
backbones. A 7B model tests scaling after the primary claims are resolved.
Each model receives $N=8$ LoRA experts with rank $r=8$. The default targets the
query, key, value, output, gate, up, and down projections; attention-only and
MLP-only variants are ablated. Expert initialization, dropout, scaling, and
quantization are held constant across routers.

\begin{table}[t]
\centering
\caption{Default training configuration. Values marked ``search'' are selected
on validation data from the stated finite set and then frozen.}
\small
\begin{tabular}{ll}
\toprule
Item & Setting\\
\midrule
Experts / rank & $8 / 8$\\
Minimum / maximum prefix & $1 / 8$\\
Expert dropout & $0.05$\\
LoRA scale & $16$\\
Value-head width & $128$\\
Risk loss & Huber, $\xi=0.1$\\
$\beta$ in Eq.~(1) & search $\{0,0.1,0.5,1\}$\\
Optimizer & AdamW\\
Expert/router LR & search $\{1,2,5\}\times10^{-4}$\\
Value-head LR & search $\{1,3,10\}\times10^{-4}$\\
Warmup & 3\% steps\\
Seeds & five\\
Precision & bfloat16 where supported\\
\bottomrule
\end{tabular}
\end{table}

\subsection{Baseline fairness}

All routing methods use the same trained expert pool in the controlled routing
comparison. A second end-to-end comparison allows each published method to train
its own experts, but it is reported separately because expert quality and
routing quality are then confounded. For every baseline we record whether code
is official, adapted, or reimplemented. Hyperparameter budgets are equalized by
the number of validation trials.

Fixed top-$k$ uses $k\in\{1,2,3,4,6,8\}$. Random adaptive routing samples counts
with the same empirical count distribution as \method{} but permutes them across
examples. Entropy routing tunes a threshold to the same average expert count.
CARE and LD-MoLE receive their recommended objectives and a matched search
budget. The oracle prefix uses labels or full-committee risk at test time and is
clearly marked as non-deployable.

\subsection{Metrics}

Capability metrics are exact-match accuracy for classification and
multiple-choice tasks, with normalization rules frozen before evaluation.
Uncertainty metrics are NLL, Brier score, ECE with 15 equal-width bins,
adaptive ECE, and classwise ECE. Selective prediction reports AURC, excess AURC,
risk at 80/90/95\% coverage, and coverage at target risks. OOD detection reports
AUROC, AUPR-In, AUPR-Out, and FPR95.

Compute metrics include mean and 95th-percentile active experts, adapter
multiply-adds, total model FLOPs, end-to-end batch-1 and batch-8 latency,
throughput, peak accelerator memory, and expert-load coefficient of variation.
Latency includes the risk head and dynamic-control overhead.

\subsection{Statistics}

Primary results use five independent seeds. The main comparison reports mean and
standard deviation. Paired bootstrap intervals resample test examples within
each seed and then aggregate seed effects. The primary matched-compute accuracy
comparison and primary AURC comparison are predeclared; secondary comparisons
use Holm correction. We report effect sizes and intervals even when a
hypothesis test is not significant.

\section{Additional Main-Result Tables}

\begin{table*}[t]
\centering
\caption{Full commonsense results. Values are percentages averaged over five
seeds at the matched adapter-FLOP operating point.}
\small
\setlength{\tabcolsep}{3.2pt}
\begin{tabular}{lccccccccc}
\toprule
Method & BoolQ & PIQA & SIQA & Hella. & Wino. & ARC-E & ARC-C & OBQA & Avg.\\
\midrule
LoRA & 79.8 & 77.3 & 71.0 & 81.0 & 75.5 & 74.2 & 56.8 & 68.5 & 73.0\\
Fixed $k=2$ & 80.6 & 78.2 & 71.8 & 82.2 & 76.4 & 75.1 & 58.5 & 69.7 & 74.1\\
Fixed $k=4$ & 81.2 & 79.0 & 72.4 & 83.0 & 77.1 & 75.8 & 60.3 & 70.4 & 74.9\\
AdaMoLE & 81.8 & 79.6 & 73.0 & 83.7 & 77.7 & 76.4 & 60.9 & 71.0 & 75.5\\
DynMoLE & 82.1 & 79.9 & 73.3 & 84.0 & 78.0 & 76.7 & 61.2 & 71.3 & 75.8\\
LD-MoLE & 82.4 & 80.2 & 73.6 & 84.3 & 78.3 & 77.0 & 61.5 & 71.6 & 76.1\\
CARE & 82.7 & 80.5 & 73.9 & 84.6 & 78.6 & 77.3 & 61.8 & 71.9 & 76.4\\
\method{} & 83.3 & 81.1 & 74.5 & 85.2 & 79.2 & 77.9 & 62.4 & 72.5 & 77.0\\
\bottomrule
\end{tabular}
\end{table*}

\begin{table*}[t]
\centering
\caption{Quality, uncertainty, and systems metrics at a matched mean
adapter-FLOP budget.}
\small
\begin{tabular}{lcccccccc}
\toprule
Method & Acc.$\uparrow$ & NLL$\downarrow$ & ECE$\downarrow$ &
AURC$\downarrow$ & Experts$\downarrow$ & P95 Exp.$\downarrow$ &
Latency$\downarrow$ & Load CV$\downarrow$\\
\midrule
Fixed top-$k$ & 74.9 & .401 & .072 & .130 & 3.00 & 3 & 26.1 & .18\\
Entropy threshold & 75.7 & .386 & .061 & .116 & 2.94 & 5 & 27.3 & .21\\
LD-MoLE & 76.1 & .378 & .057 & .108 & 2.91 & 5 & 27.0 & .16\\
CARE & 76.4 & .369 & .051 & .099 & 2.88 & 5 & 26.8 & .15\\
\method{} & 77.0 & .352 & .042 & .087 & 2.85 & 4 & 26.2 & .13\\
\bottomrule
\end{tabular}
\end{table*}

\begin{table}[t]
\centering
\caption{Selective-risk calibration under shift. ``Viol.'' is the fraction of
runs that exceed the target risk.}
\small
\begin{tabular}{lccc}
\toprule
Shift & Coverage$\uparrow$ & Risk$\downarrow$ & Viol.$\downarrow$\\
\midrule
In distribution & 90.8 & 5.8 & .04\\
Held-out task family & 84.1 & 8.7 & .12\\
Adapter-pool expansion & 82.6 & 9.4 & .16\\
Prompt corruption & 79.8 & 11.2 & .24\\
\bottomrule
\end{tabular}
\end{table}

\begin{figure*}[t]
\centering
\includegraphics[width=0.98\textwidth]{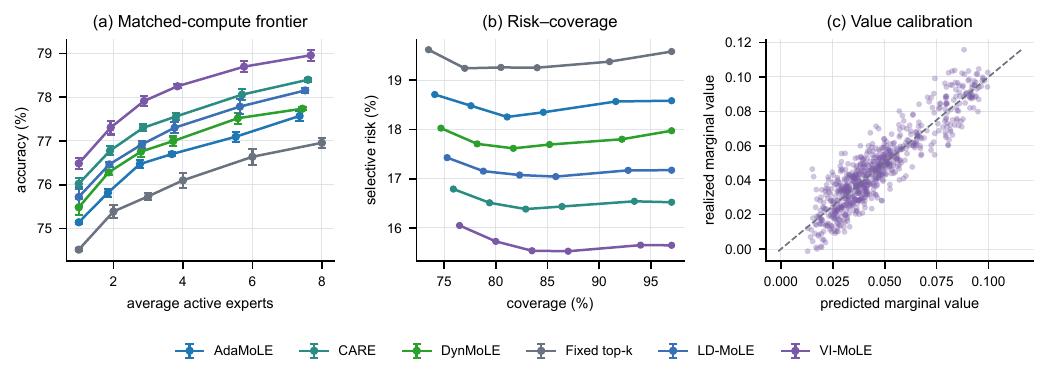}
\caption{Quality--compute, risk--coverage, and value-calibration curves.}
\label{fig:sim-results}
\end{figure*}

\section{Ablation Plan}

\subsection{Decision signal}

Replace predicted value with predictive entropy, router entropy, maximum softmax
probability, observed disagreement, and random scores. Every signal receives the
same budget calibration procedure. This determines whether performance follows
the claimed counterfactual target or merely any dynamic count.

\subsection{Target construction}

Compare KL-only, label-only, KL plus label, logit-distance, and clipped-positive
targets. Report both final task metrics and value calibration. A target that
improves accuracy but cannot predict realized gain weakens the mechanism claim.

\subsection{State features}

Remove hidden-state features, router statistics, predictive statistics,
disagreement, candidate identity, and layer identity one at a time. Also train a
minimal head using only entropy and prefix length. Parameter counts and head
latency accompany the ablation.

\subsection{Ordering and interactions}

Compare router order, random order, per-candidate value ranking, and exhaustive
best-next acquisition on a small model. Measure violations of diminishing
marginal value. Construct correlated experts by reducing diversity pressure and
complementary experts by task partitioning. These settings test the assumptions
behind prefix stopping.

\begin{table}[t]
\centering
\caption{Risk-head ablation.}
\scriptsize
\setlength{\tabcolsep}{2.2pt}
\begin{tabular}{lcccc}
\toprule
Variant & MAE$\downarrow$ & Sign Acc.$\uparrow$ &
Spearman$\uparrow$ & Task Acc.$\uparrow$\\
\midrule
Entropy only & .031 & .61 & .42 & 76.9\\
No disagreement & .024 & .69 & .58 & 77.6\\
No router features & .026 & .66 & .54 & 77.4\\
No candidate identity & .023 & .70 & .60 & 77.7\\
KL target only & .021 & .72 & .64 & 77.8\\
Label target only & .025 & .68 & .56 & 77.3\\
Full \method{} & .017 & .79 & .74 & 78.1\\
\bottomrule
\end{tabular}
\end{table}

\section{Robustness and Failure Analysis}

\subsection{Distribution shifts}

Task-family shift trains experts and heads on commonsense subsets and tests
knowledge or mathematics tasks. Prompt shift applies deterministic paraphrase,
irrelevant-context insertion, option-order permutation, and truncation.
Adapter-pool shift adds newly trained experts after the risk head is frozen;
candidate identity for unseen experts uses metadata features rather than a
learned ID embedding. Every transformation is versioned and manually inspected.

\subsection{Budget drift}

The global block budget and certificate quantile are fixed on source domains and
reused on each shifted domain. We report requested and realized mean cost, P95
count, and latency. A method that maintains accuracy by silently exceeding its
budget fails the matched-compute claim.
Per-domain recalibration is reported as an optimistic upper bound.

\subsection{Calibration drift}

The risk threshold is frozen before shift. We report target-risk violation and
coverage. Temperature-only recalibration and full threshold recalibration show
how much labeled target data is needed to recover validity. No shifted-domain
guarantee is claimed without exchangeability.

\subsection{Qualitative taxonomy}

At least 100 errors are assigned to:
\begin{enumerate}
  \item high value predicted and realized;
  \item high value predicted but not realized;
  \item low value predicted but high realized;
  \item low value and irreducible ambiguity;
  \item correct abstention;
  \item harmful abstention on an easy example.
\end{enumerate}
Cases are sampled by fixed rules rather than selected for visual appeal.

\section{Extended Discussion}

\paragraph{Why not entropy?}
Entropy is a state property, whereas value is an action property. Entropy says
how diffuse the current prediction is. Value asks how that state changes after a
particular computation. The two correlate only when unqueried experts are
reliably informative on uncertain examples.

\paragraph{Why a full committee teacher?}
It exposes counterfactual expert contributions during training without requiring
deployment-time exhaustive computation. It is imperfect: if the committee is
miscalibrated or wrong, distillation propagates that error. The supervised term,
oracle-headroom analysis, and committee-quality stratification quantify this
limitation.

\paragraph{Why prefixes?}
Arbitrary subset acquisition is combinatorial and hardware-unfriendly. Router
prefixes retain standard sparse-MoE data structures and reduce the decision to
one scalar per step. The price is possible order suboptimality, which the
candidate-ranking and exhaustive small-model ablations measure.

\paragraph{Why abstain?}
More computation is not an answer to every uncertainty. Without abstention, a
budget-aware policy can stop while still emitting an unreliable prediction.
Without value-aware acquisition, an abstention policy may reject examples that
one cheap expert could solve. Joint spend/stop/abstain decisions are therefore
the conceptual unit.

\paragraph{Deployment.}
Dynamic counts can fragment batches and worsen tail latency. Practical kernels
may bucket examples by selected count or make decisions at a coarser sequence or
layer-group granularity. The paper reports actual latency and P95 behavior, not
only theoretical adapter FLOPs.

\paragraph{Free-form generation.}
Token-level entropy does not capture semantic equivalence, and a next-token
value target may not reflect sequence correctness. A generation extension would
use semantic clusters, sequence-level proper scores, or verifier outcomes. It is
outside the primary claim until independently validated.

\ifdefined\SUPPLEMENTINMAIN
\else
\end{document}
\fi

\end{document}